\documentclass[letterpaper, 10 pt, conference]{IEEEconf}
\IEEEoverridecommandlockouts

\usepackage{microtype}
\usepackage{amsmath,amssymb,amsfonts,bm,mathtools}
\usepackage{mathrsfs}
\usepackage{graphicx}
\usepackage[usenames,dvipsnames]{xcolor}
\usepackage[colorlinks=true,linkcolor=MidnightBlue,citecolor=OliveGreen,urlcolor=RoyalBlue]{hyperref}
\usepackage[nameinlink]{cleveref}
\usepackage{comment}

\usepackage{amsmath, amssymb}   
\usepackage{algorithm}         
\usepackage{algpseudocode}     

\newtheorem{theorem}{Theorem}
\newtheorem{lemma}[theorem]{Lemma}

\newcommand{\reals}{\mathbb{R}}
\newcommand{\E}{\mathbb{E}}

\newcommand{\1}{\mathbf{1}}

\newcommand{\CE}{\mathbb{E}}    

\newcommand{\CF}{\mathcal{F}}

\newcommand{\CL}{\mathcal{L}}

\newcommand{\skor}{\mathcal{S}}

\newcommand{\ones}{\mathbf{1}}

\renewcommand{\leq}{\leqslant}
\renewcommand{\geq}{\geqslant}
\renewcommand{\epsilon}{\varepsilon}
\renewcommand{\phi}{\varphi}
\renewcommand{\th}{\theta}

\newcommand{\loss}{\ell}
\newcommand{\Xt}{X^\th}
\newcommand{\Loss}{L}

\begin{document}

\title{\LARGE \bf Malliavin Calculus for Counterfactual Gradient Estimation in Adaptive Inverse Reinforcement Learning}

\author{{Vikram Krishnamurthy},  {Luke Snow}
\thanks{This research was supported in part by the National Science Foundation grants  CCF-2312198 and  Army Research Office grant  W911NF-24-1-0083.}
\thanks{Luke Snow {\tt\small las474@cornell.edu}, and Vikram Krishnamurthy {\tt\small vikramk@cornell.edu} are  with the School of Electrical and Computer Engineering, Cornell University, Ithaca, NY 14853, USA
        }%
}

\maketitle

\begin{abstract}
Inverse reinforcement learning (IRL) recovers the loss function of a forward learner from its observed responses. Adaptive IRL aims to reconstruct the loss function of a forward learner by passively observing its gradients as it performs reinforcement learning (RL). This paper proposes a novel passive Langevin-based algorithm that achieves adaptive IRL. The key difficulty in adaptive IRL is that the required gradients in the passive algorithm are counterfactual: they are conditioned on events of probability zero under the forward learner’s trajectory. Therefore, naive Monte Carlo estimators are prohibitively inefficient, and kernel smoothing, though common, suffers from slow convergence. We overcome this by employing Malliavin calculus to efficiently estimate the required counterfactual gradients; we reformulate the counterfactual conditioning as a ratio of unconditioned expectations involving Malliavin quantities, thus recovering standard estimation rates. We derive the necessary Malliavin derivatives and their adjoint Skorohod integral formulations for a general Langevin structure, and provide a concrete algorithmic approach which exploits these for counterfactual gradient estimation. 


\end{abstract}

\section{Introduction}
Inverse reinforcement learning (IRL) aims to estimate the loss function of optimizing agents by observing their actions (estimates). Classical IRL is off-line: given a data set of actions chosen according to the optimal policy of a Markov decision process, \cite{NR00} formulated a set of inequalities that the loss function must satisfy. This paper constructs and analyzes an adaptive (real time) IRL algorithm by observing optimizing agents that are performing real time reinforcement learning (RL). The problem we consider is this: Suppose we observe estimates of multiple (randomly initialized) stochastic gradient algorithms (reinforcement learners) that  maximize a (possibly non-convex) loss. How to design another stochastic gradient algorithm (inverse learner) to estimate the loss function?


Existing techniques for adaptive IRL \cite{KY21,KY22,SK25} incorporate the agent's learning dynamics through a passive \emph{kernel-based} passive Langevin dynamics, which converges asymptotically to a stationary measure encoding the loss function $\Loss$ to be learned; this allows for adaptive IRL through Markov chain Monte Carlo (MCMC) sampling of the stationary measure of this \emph{passive} Langevin chain. Similar passive schemes and stochastic approximation analyses have been investigated in \cite{KY22}, \cite{YY96}, \cite{KY03}, \cite{NPT89}. 

In this work we propose an enhanced alternative to these kernel-based passive Langevin algorithms: we derive a Malliavin calculus-based counterfactual gradient estimate which replaces the passive kernel. This allows us to bypass kernel bandwidth scale limitations and compute the necessary Langevin gradient directly from the \emph{mis-specified} dynamics of the learning agent. In this sense we evaluate a \emph{counterfactual} gradient\footnote{Counterfactual gradient estimation asks: Given sample paths of $X_t$ that we cannot control, how can we evaluate $\nabla\Loss(X_t)$ under the hypothetical condition that the sample paths $X_t$ pass through specific point $\alpha$ \cite{KS25}.}: we observe only gradient dynamics from the agent's learning trajectory and we compute from this the loss gradient at a \emph{distinct} point which is not observed in the trajectory. This counterfactual estimate is formulated as an expected gradient conditioned on the measure-zero event that the learning trajectory passes through our distinct evaluation point; this conditioning prohibits classical Monte Carlo since the conditioning event has zero probability. We exploit tools from Malliavin calculus to reformulate this conditional expected gradient as a ratio of \emph{unconditioned} expectations involving Malliavin derivatives and their adjoint Skorohod integrals of the observed learning dynamics; thus, recovering classical Monte Carlo estimation rates even in this rare or measure-zero event regime. This efficient reformulation is enabled by \emph{exploitation} of the structure of the learning dynamics through the Malliavin quantities. 

This Malliavin-based gradient estimator then enables adaptive IRL by replacing the kernel-based Langevin gradient in existing techniques \cite{KY21,SK25}. Compared to sequential MCMC approaches, which suffer from particle degeneracy, require large ensembles, and exhibit variance explosion over long time horizons, the Malliavin calculus method of \cite{BET04} yields unbiased Monte Carlo estimators for the counterfactual conditional expectations that achieve the optimal $\sqrt{N}$ convergence rate independent trajectories of the Langevin dynamics, without resampling or kernel smoothing.

To give additional context, estimating a loss function given the response of agents is studied in the area of revealed preferences in microeconomics  \cite{Afr72,WD12}. In comparison, the current paper uses misspecified  gradients to recover the loss function  in real time via a  Langevin  algorithm.



The paper is organized as follows:  Section~\ref{sec:ad_IRL}  introduces our approach to adaptive IRL.  Section~\ref{sec:mall_intro}  provides an introduction to Malliavin calculus, and presents the key conditional expectation reformulation which we exploit to achieve adaptive IRL. In Section~\ref{sec:mall_pIRL} we derive the full Malliavin-based algorithmic approach to adaptive IRL.  Section~\ref{sec:numerical_example}  provides a numerical implementation of our algorithm, demonstrating effective counterfactual gradient computation and recovery of the forward learner's loss function.

\begin{table}[!t]
\caption{Notation}
\label{tab:notation}
\centering
\renewcommand{\arraystretch}{1.15}
\begin{tabular}{@{}ll@{}}
\hline
\textbf{Symbol} & \textbf{Meaning} \\
\hline
$L$ & unknown loss of the forward learner \\
$X_t$, $X_k$ & forward Langevin state (continuous / discrete) \\
$\alpha$ & query point; state of the inverse Langevin chain \\
$\hat{\nabla} L(\alpha)$ & counterfactual estimate of $\nabla L$ at $\alpha$ \\
$\beta$ & inverse temperature; $\pi_\infty\propto e^{-\beta L}$ \\
$\varepsilon$, $\eta$ & forward / inverse step sizes \\
$s$ & evaluation time in $\mathbb{E}[\nabla L(X_s)\mid X_s=\alpha]$ \\
$D_t F$ & Malliavin derivative of $F$ at time $t$ \\
$S(u)$ & Skorohod integral of the weight $u$ \\
$\delta$ & Dirac delta  \\
\hline
\end{tabular}
\end{table}

\section{Forward and Inverse Learner Models}
\label{sec:ad_IRL}

We describe the adaptive IRL problem from two points of view: The forward learner that comprises multiple agents performing reinforcement learning (RL). These agents act sequentially to perform RL by using stochastic gradient
algorithms to minimize a loss function.
Second,  we discuss the inverse learner that deploys a passive Langevin dynamics algorithm. We also motivate the counterfactual gradient estimation problem for the inverse learner. 



\subsection{Forward Learning Dynamics}
We model the forward learner as a discrete stochastic gradient Langevin dynamics (SGLD) with episodic re-initialization, as in \cite{KY21,SK25}. Let
\[
0=\tau_0<\tau_1<\tau_2<\cdots
\]
denote stopping times, and suppose that at the beginning of each episode,
\[
X_{\tau_i}\sim \pi_0.
\]
For \(k\in\{\tau_i,\tau_i+1,\dots,\tau_{i+1}-1\}\), let $X_k\in\reals^d$ follow
\begin{align}
\begin{split}
\label{eq:fwd_lrn}
    X_{k+1}
&=
X_k-\varepsilon\,\widehat{\nabla}\Loss(X_k)
+\sqrt{2\varepsilon}\,\xi_{k+1},\\
\xi_{k+1}&\sim \mathcal N(0,I),
\end{split}
\end{align}
where \(\widehat{\nabla}\Loss(X_k)\) is a noisy gradient estimate. This is an Euler-Maruyama discretization of the continuous Langevin diffusion. Indeed, under the standard two-time-scale averaging assumptions \cite{KY21}, the piecewise-constant interpolation of \((X_k)\) converges weakly, as \(\varepsilon\to 0\), to the Langevin diffusion
\begin{equation}
\label{eq:langevin}
dX_t=-\nabla \Loss(X_t)\,dt+\sqrt{2}\,dW_t,
\qquad t\in[\tau_i,\tau_{i+1}),
\end{equation}
with re-initialization \(X_{\tau_i}\sim \pi_0\) at each stopping time, and where $W$ denotes $d$-dimensional standard Brownian motion. Thus, for the remainder we consider \eqref{eq:langevin} as the continuous-time forward learning process, capturing the asymptotic slow time-scale behavior of its discrete SGLD recursion. 

The forward learner is designed to generate informative gradient samples across the state space rather than solely to converge to a minimizer of $\Loss$. To this end, we employ gradient-based dynamics augmented with injected Langevin noise and random restarts. These mechanisms serve complementary roles: local stochastic perturbations regularize the dynamics and prevent collapse onto deterministic trajectories, while restarts provide occasional global reinitialization, ensuring sustained exploration over long horizons. This allows for  global nonconvex optimization \cite{RRT17a} and Bayesian posterior sampling \cite{WT11}, for instance.

\subsection{Inverse Learning Dynamics}
Next, 
we  adopt the perspective of the \emph{inverse learner}. The inverse learner observes the forward process
\eqref{eq:langevin}  passively: the trajectory of the forward learner is not controlled, but instead the inverse learner seeks to reconstruct the unknown loss function $L(\cdot)$
from samples provided by the forward learner. The goal of the inverse learner is not merely to estimate the loss
$L$ at a single point, but to recover the loss function $L(\alpha)$ over the entire state
space.

The key idea is to estimate, for any chosen query point $\alpha \in \mathbb{R}^d$,
the gradient $\hat{\nabla} L(\alpha)$ using only the observed forward trajectories.
Once such pointwise gradient estimates are available, we may construct an
\emph{inverse} Langevin dynamics on the variable $\alpha$:
\begin{equation}
\label{eq:inv_langevin}
d\alpha_t = -\hat{\nabla} L(\alpha_t)\,dt + \sqrt{2\beta^{-1}}\,dW_t,
\end{equation}
where $\beta > 0$ is an inverse-temperature parameter chosen by the designer.

Under standard conditions (see Theorem~\ref{thm:convg}), the diffusion \eqref{eq:inv_langevin} admits the
Gibbs measure
\begin{equation}
\pi_\infty(\cdot) \propto \exp\bigl(-\beta L(\cdot)\bigr)
\label{eq:gibbs_measure}
\end{equation}
as invariant distribution. Therefore, if we can simulate
\eqref{eq:inv_langevin}, or a consistent discretization of it, using
estimated gradients, then samples from its stationary regime encode the unknown
loss function itself. In particular, 
\begin{equation}
L(\alpha) \propto -\frac{1}{\beta}\log \pi_\infty(\alpha) .
\label{eq:recover_L_from_gibbs}
\end{equation}


\subsection{Estimating the Counterfactual Gradient}

The central difficulty for the inverse learner is that the gradient needed in
\eqref{eq:inv_langevin} is \emph{counterfactual}. The forward learner
generates trajectories $(X_t)_{t\geq 0}$ according to
\eqref{eq:langevin}, but the inverse learner requires the value of
$\nabla L(\alpha)$ at a designer-chosen point $\alpha$, which need not lie on
the observed trajectory.

To implement the adaptive IRL, we will compute 
\begin{equation}
\label{eq:loss_grad}
\nabla L(\alpha) = \E[\nabla L(X_s)\mid X_s=\alpha].
\end{equation}
for some fixed evaluation time $s \in [0,\tau]$. This equation can be re-expressed as the ratio 
\begin{equation}
\label{eq:dirac_ratio}
\nabla\Loss(\alpha) = \frac{\CE[\nabla\Loss(X_s)\delta(X_s-\alpha)]}{\CE[\delta(X_s-\alpha)]}
\end{equation}
where $\delta$ is the Dirac delta function. 


However, the event $\{X_s=\alpha\}$ has probability zero, so direct conditional
Monte Carlo is infeasible. Existing passive IRL methods address this using
kernel localization around $\alpha$, but such estimates suffer from bandwidth
selection and slow convergence in low-probability regions.

In this work, we instead use the integration-by-parts formula in Malliavin calculus to integrate the delta function in \eqref{eq:dirac_ratio} to a step function and simultaneously Malliavin-differentiate $\nabla\Loss(X_s)$.
This removes the singular conditioning event from the estimator and yields a
computable expression based on pathwise quantities of the observed forward
diffusion. These gradient estimates are then used in the outer Langevin dynamics
\eqref{eq:inv_langevin} to recover the Gibbs law \eqref{eq:gibbs_measure},
and hence the loss $L(\cdot)$. 

To summarize, the outcome of our Malliavin calculus approach is that we construct an efficient nonparametric estimate of the counterfactual gradient \eqref{eq:loss_grad}, which enables adaptive IRL through the Langevin dynamics \eqref{eq:inv_langevin}.

\section{Malliavin Calculus Approach to Estimate Conditional Loss}
\label{sec:mall_intro}

In this paper, to achieve adaptive IRL, we will  compute the conditional expectation \eqref{eq:loss_grad} efficiently using tools in Malliavin calculus. This section briefly introduces the necessary tools.

Malliavin calculus \cite{Nua06} was  developed in the 1970s as a probabilistic method to prove H\"ormander's theorem for the solution of SDEs. It was later adapted in mathematical finance to compute sensitivities (Greeks) of option prices.  Here, as in \cite{FLL01,BET04,Cri10}, we employ  Malliavin calculus  to \textit{efficiently} evaluate the conditional loss $\E\{\loss(\Xt)|g(\Xt)=0\}$, even when the conditioning event $\{g(\Xt)=0\}$ has vanishing or zero probability. 


\subsection{Preliminaries. Malliavin Calculus.}
\label{sec:mall_pre}

We recall the two central objects.

\subsubsection{Malliavin derivative}
 Consider the probability space $(\Omega,\mathcal{F},\mathbb{P})$ with  a $d$-dimensional Brownian motion $W = (W^1,\dots,W^d)$ and the natural filtration $\{\CF_t\}_{t\geq0}$. 
For a smooth functional $F$ of $W$, the \emph{Malliavin derivative} $D_t F$ is defined as the process measuring 
the infinitesimal sensitivity of $F$ to perturbations of the Brownian path at time $t$. Formally, 
for cylindrical random variables of the form
\[
F = f\bigg( \int_0^T h_1(s)\, dW_s, \dots, \int_0^T h_n(s)\, dW_s \bigg),
\]
with $f \in C_b^\infty(\mathbb{R}^n)$ and $h_i \in L^2([0,T];\mathbb{R}^d)$, the derivative is
\[
D_t F = \sum_{i=1}^n \frac{\partial f}{\partial x_i}\bigg( \int_0^T h_1\, dW, \dots, \int_0^T h_n\, dW \bigg)\, h_i(t).
\]
The closure of this operator in $L^p$ leads to the Sobolev space $\mathbb{D}^{1,p}$ of Malliavin differentiable random variables.

\subsubsection{Skorohod integral}
The adjoint of the Malliavin derivative is the \emph{Skorohod integral}, denoted $\skor(u)$.  Indeed, for a process $u \in L^2([0,T]\times \Omega;\mathbb{R}^d)$, $u$ is in the domain of $\skor$ if there exists a square-integrable 
random variable $\skor(u)$ such that for all $F \in \mathbb{D}^{1,2}$,
\begin{equation}
\label{eq:adjoint}
\E[F \, \skor(u)] = \E\!\left[\int_0^T \langle D_t F, u_t \rangle \, dt \right].
\end{equation}
The above adjoint relationship  serves as the definition of the Skorohod integral, and can be written abstractly as  $$\langle F, \mathcal{S}(u) \rangle_{L^2(\Omega)} =
\langle D F, u \rangle_{L^2([0,T] \times \Omega)}.$$

When $u$ is adapted to the filtration $\{\CF_t\}_{t\geq0}$, the Skorohod integral $\skor(u)$ coincides with the Itô integral $\int_0^T u_t\, dW_t$. 
In general, $\skor(u)$ extends stochastic integration to non-adapted processes and is sometimes called the \emph{divergence operator}.

\subsubsection{Integration by parts}
The duality relation~\eqref{eq:adjoint} yields the Malliavin integration-by-parts formula, which underpins many applications, 
including Monte Carlo estimation of conditional expectations and sensitivity analysis for SDEs 
(see \cite{Nua06,FLL01,BET04}). 

\subsubsection{Computing Malliavin Derivative and Skorohod Integral} 
\label{sec:mall_comp}
The following properties are the key tools which allow us to compute the Malliavin derivative and Skorohod integral: 
\begin{enumerate}
    \item \textit{Malliavin derivative of diffusion}. For $\{X_t\}_{t\geq 0}$ a diffusion process, the  Malliavin derivative $D_sX_t$ is \cite{GM05}
    \begin{equation}
    \label{eq:mall_form}
        D_sX_t = Y_tZ_s\sigma(X_s,s)\1_{s\leq t}
    \end{equation}
    where  $Y_t:= \nabla_x X_t$ is the Jacobian matrix and $Z_t$ is its inverse $Z_t := Y_t^{-1}$.  This, together  with the Malliavin chain rule \cite{Nua06}, facilitates evaluating  Malliavin derivatives of general functions of diffusions.
    \item \textit{Skorohod expansion}. For random variable $F\in \mathbb{D}^{1,2}$ and Skorohod-integrable process $v$, we have \cite[eq. 2.2]{GM05} (we use $\skor$ instead of $\delta$ for the Skorohod integral to avoid notational clash with the Dirac delta function):
    \begin{equation}
    \label{eq:skor_exp}
    \skor(F v) = F\skor(v) - \int_0^TD_tF\cdot v_t dt
    \end{equation}
    In general,  the Skorohod integrand $\{u_t\}_{t\in[0,T]}$ of interest may be  non-adapted. However, in the special case where $u$ factorizes into  the product of an adapted process $v=\{v_t\}_{t\in[0,T]}$ and an anticipative random variable $F$, this formula gives a constructive expression.  Specifically, we can expand $\skor(u) = \skor(Fv) $ using  \eqref{eq:skor_exp} and compute it in terms of a  standard It\'o integral of the adapted part $v$ together with the  Malliavin derivative of the anticipatory random variable $F$. 
    
\end{enumerate}

\subsection{Malliavin Calculus  for Conditional Expectation}
In this section we present the general conditional expectation reformulation using Malliavin calculus; we will exploit this in Section~\ref{sec:mall_pIRL} as the key tool enabling our passive adaptive IRL algorithm. 

The following main result expresses the conditional expectation \eqref{eq:loss_grad} as the ratio of unconditional expectations. 
\begin{theorem}  
\label{thm:mall_ce} Assume
$\ell(\Xt), g(\Xt) \in L^2(\Omega)$ and  $D_t\ell(\Xt), D_tg(\Xt) \in L^2(\Omega \times [0,T])$. Then the following conditional expectation reformulation holds
\begin{align}
\begin{split}
\label{eq:malliavin}
 &\E[ \ell(X^\th)  \mid g(X^\th) = 0]  =  \frac{E_1^\th}{E_2^\th} \\
   &\text{ where } \\&E_1^\th 
 = \E\bigg[ \ones_{\{g(\Xt)>0\}}\left(\ell(X^\th) \skor(u)  - 
  \int_0^T (D_t \ell(\Xt)) u_t dt\right) \bigg]\\  &E_2^\th =
\E[  \ones_{\{g(\Xt)>0\}} \skor(u)]
\end{split}
\end{align}
Here 
$u$ is any process that satisfies  
\begin{equation}
\label{eq:ut_cond}
\int_0^T D_tg(\Xt) u_t = 1 \quad \text{almost surely}
\end{equation}
\end{theorem}

\textbf{Proof outline}: We start with \eqref{eq:malliavin} and write $\skor(g(\Xt))$ as $\skor(G)$.
Then, by the Malliavin chain rule, the adjoint relation \eqref{eq:adjoint} and the Skorohod integrand condition \eqref{eq:ut_cond}, we have
\begin{align*}
&\E[\ell(\Xt)\,\skor(G)] \\&
=\E\!\left[\int_0^T (D_t(\ell(\Xt))\1_{\{g(\Xt)>0\}}))\,u_t\,dt\right]
\\&= \E\!\Big[\1_{\{g(\Xt)>0\}}\left(\ell(\Xt)\skor(u)-
\int_0^T (D_t\ell(\Xt))u_t\,dt\right)\Big].
\end{align*}
The denominator in \eqref{eq:malliavin} can be derived similarly. 

\paragraph*{Remarks}
(i) There is considerable flexibility  in the choice of $u$ in the above theorem.
The canonical choice is:
For $g(\Xt)  \in \mathbb{D}^{1,2}$ with Malliavin derivative 
$D g(\Xt) = \{ D_t g(\Xt) \}_{t \in [0,T]} \in L^2(\Omega; H)$, 
choose (see \cite{Nua06})
\begin{equation}
  \label{eq:generalu}
u_t = \frac{D_t g(\Xt)}{\|D g(\Xt)\|_H^2}, \qquad t \in [0,T]
\end{equation}
where $H$ is the Cameron--Martin space with norm
\[
\|h\|_H^2 := \int_0^T |h(t)|^2\,dt.
\]
The  choice~\eqref{eq:generalu} ensures that $u \in H$ and is always well-defined. 
However, in certain special cases one may use simpler (though less general) 
expressions. 
For example if $D_t g(\Xt) \neq 0$ a.e., one can choose
\begin{equation}
\label{eq:u_choice}
u_t = \begin{cases} \frac{1}{T D_t g(\Xt)}  & D_t g(\Xt) \neq 0 \\
     1 &  D_t g(\Xt) = 0 .
     \end{cases}
\end{equation}
But one has to be careful with the choice~\eqref{eq:u_choice}. For $g(\Xt) = \int_0^t W_s ds$,
then $D_t g(\Xt)  = T-t$ so that  $u_t = \frac{1}{T D_t g(\Xt)} = \frac{1}{T(T-t)}  \notin H$.
In comparison, choosing $u$ according to \eqref{eq:generalu} yields $u_t = 3(T-t)/T^3 \in H$.

(ii) 
The 
representation~\eqref{eq:malliavin} requires evaluation of Malliavin derivatives and Skorohod integrals, see \cite{GM05} for several examples. There are  several important consequences.  First, it restores the
$N^{-1/2}$ Monte--Carlo convergence rate even under singular conditioning, as
the event $\{g(X^\theta)=0\}$ no longer needs to be sampled directly.  Second,
the estimator admits substantial variance--reduction flexibility: the choice of
localizing function (indicator versus smooth approximation) and of admissible
weight process $u$ strongly influence efficiency, with optimal choices
characterizable via variational principles in Malliavin calculus.  Third, the
representation is compatible with standard discretizations of the forward SDE:
the Malliavin derivatives $D_t X^\theta$ admit recursive Euler--Maruyama
approximations, so one avoids additional kernel bandwidths or curse--of--dimensionality
issues inherent in regression--based methods.

\section{Malliavin Estimator in Passive IRL Algorithm}
\label{sec:mall_pIRL}

In this section we detail how the Malliavin calculus tools introduced in Section~\ref{sec:mall_intro} can be used to evaluate the gradient $\nabla\Loss(\alpha)$ of the loss function $\Loss(\alpha)$ (unknown to the inverse learner) at a point $\alpha$ given the forward learner's evaluation \eqref{eq:langevin} at a different point $X_s$. This counterfactual gradient evaluation is the key component allowing us to perform passive IRL: we embed this gradient estimate in a slower time-scale passive Langevin chain, from which we can recover $\Loss(\cdot)$ in its entirety by MCMC sampling. 

For simplicity of notation and clarity of presentation, in this section we consider a one dimensional diffusion \eqref{eq:langevin}. Multi-dimensional analogues are handled similarly, see \cite{BET04} or \cite{KY09}.

\subsection{Computing conditional gradient \eqref{eq:loss_grad}.}

Recall that the key to our passive learning procedure lies in computing the \emph{counterfactual gradient} \eqref{eq:loss_grad} efficiently. We exploit the representation from Theorem~\ref{thm:mall_ce} as:
\begin{align}
\begin{split}
\label{eq:mall_reform}
&L'(\alpha) = \CE\!\left[ \Loss'(X_s)\mid X_s=\alpha\right] \\&\,\,= \frac{\CE\left[\1_{X_s >\alpha}\left(\Loss'(X_s)\skor(u) - \langle D(\Loss'(X_s)),u\rangle\right)\right]}{\CE\left[\1_{X_s > \alpha}\skor(u)\right]}
\end{split}
\end{align}
where 
\begin{equation}
\label{eq:u_comp}
u_t = \begin{cases}
\frac{1}{TD_tX_s}, \,\,&D_tX_s \neq 0\\
1, &D_tX_s = 0
\end{cases} \quad s\in(0,T]
\end{equation}
and $s$ is a designer-specified evaluation time.

Recall that the sample paths from \eqref{eq:langevin} we (as the inverse learner) observe are 
\begin{equation}
\label{eq:X_sps}
\{X_t, \Loss'(X_t), \Loss''(X_t)\}_{t\in[0,\tau_i]}, \,\,\,i=0,1,\dots
\end{equation}
where $\tau_i$ are measurable stopping times. 

\subsection{Malliavin Derivative and Skorohod Integral for Langevin Dynamics}

We now present explicit formulations of the necessary Malliavin quantities arising in the numerator and denominator of \eqref{eq:mall_reform}, which are computable from \eqref{eq:X_sps}. 
Specifically, we compute the quantities $D_tX_s$, $u_t$, and $\skor(u)$ appearing in
\eqref{eq:mall_reform} and \eqref{eq:u_comp}, exploiting the Langevin structure of the forward process
\eqref{eq:langevin}. 

\begin{lemma}[Malliavin derivative of the Langevin Diffusion]
\label{lem:DXs}
Let $(X_t)_{t\geq 0}$ satisfy
\begin{equation}
\label{eq:langevin_recall}
dX_t = -\Loss'(X_t)\,dt + \sqrt{2}\,dW_t,
\qquad X_0=x.
\end{equation}
Then, for $0\leq t\leq s$,
\begin{equation}
\label{eq:DtXs}
D_tX_s
=
\sqrt{2}\exp\!\left(-\int_t^s \Loss''(X_u)\,du\right)\1_{\{t\le s\}}.
\end{equation}
\end{lemma}
\begin{proof}
See Appendix~\ref{sec:pf_DXs}
\end{proof}

To apply the Malliavin reformulation \eqref{eq:mall_reform}, we require a
process $u$ satisfying
\[
\int_0^T D_tX_s\,u_t\,dt = 1
\qquad \text{a.s.}
\]
A convenient choice is the following.

\begin{lemma}[Choice of Skorohod Integrand]
\label{lem:u_choice}
The choice of $u_t$ in \eqref{eq:u_comp} can now be written explicitly as
\begin{equation}
\label{eq:u_choice_2}
u_t
=
\frac{1}{\sqrt{2}T}
\exp\!\left(\int_t^s \Loss''(X_u)\,du\right),
\qquad 0\le t\le T.
\end{equation}
Equivalently, $u_t=\Gamma v_t$ with
\begin{align}
\begin{split}
\label{eq:u_decomp}
\Gamma&=\exp\!\left(\int_0^s \Loss''(X_u)\,du\right),\\
v_t&=\frac{1}{\sqrt{2}T}\exp\!\left(-\int_0^t \Loss''(X_u)\,du\right).
\end{split}
\end{align}
Here $\Gamma$ is anticipative, while $v$ is adapted.
\end{lemma}

\begin{proof}
This follows immediately.
\end{proof}

\begin{algorithm*}[h]
\caption{Malliavin-based Adaptive Inverse Reinforcement Learning}\label{alg:mall_IRL}
\begin{algorithmic}[1]
\State Forward Learner: $dX_t = - \Loss'(X_t)dt + \sqrt{2} dW_t, \quad t \in [0,\tau_i],\quad \tau_i \text{  stopping time}, \quad i=1,2,\dots$
\State Observable Information: $\{X_t, \Loss'(X_t), \nabla^2\Loss(X_t)\}_{t\in[0,\tau_i]}, \quad i=1,2,\dots$

\State Initialize $\alpha_0 \in \reals$, step-size $\eta$, inverse temperature $\beta$. 
\For{$k\in[0,K]$} {\algorithmiccomment{passive Langevin chain discrete updates}}
\For{$i \in [0,\mathcal{N}]$} {\algorithmiccomment{gradient estimator Monte-Carlo samples}}
\State Observe $\{X_t,\Loss'(X_t),\Loss''(X_t)\}_{t\in[0,\tau_i]}$ {\algorithmiccomment{Observation of mis-specified sample paths}}
\State Compute $\skor(u)$ as in Lemma~\ref{lem:skor}. 
\State Compute $D(\Loss'(X_s)),u\rangle = \int_0^s  \Loss''(X_s)D_tX_su_tdt$, with $D_tX_s$ in \eqref{eq:DtXs}. 
\State Evaluate numerator $N_i$ and denominator $D_i$ of \eqref{eq:mall_reform} with $\alpha = \alpha_k$ {\algorithmiccomment{Counterfactual gradient estimator}}

\EndFor

\State Compute $N = \frac{1}{\mathcal{N}}\sum_{i=1}^{\mathcal{N}}N_i$, $D = \frac{1}{\mathcal{N}}\sum_{i=1}^{\mathcal{N}}D_i$ as Monte-Carlo estimates of  \eqref{eq:mall_reform}.

\State Assign $\widehat{\Loss'}(\alpha_k) = N/D$

\State $\alpha_{k+1} = \alpha_k - \eta \widehat{\Loss'}(\alpha_k) + \sqrt{2\eta \beta^{-1}}w_k$, \quad $w_k \sim \mathcal{N}(0,1)$ {\algorithmiccomment{$\Rightarrow \,\,\,\pi_{\infty}\propto \exp(-\beta \Loss(\cdot))$}}

\If {$k \geq  \hat{k}$}: {\algorithmiccomment{$\hat{k}$ burn-in time}}
\State $\hat{\pi}$.append($\alpha_k$) {\algorithmiccomment{Empirical histogram $\hat{\pi}$ of MCMC samples $\alpha_k$}}
\EndIf
\EndFor

\State $\hat{L}(\cdot) = -\frac{1}{\beta}\log(\hat{\pi}(\cdot))$ {\algorithmiccomment{Final loss estimate}}
\end{algorithmic}
\end{algorithm*}

\subsection{Computing Skorohod integral $\skor(u)$} Now, by \eqref{eq:u_decomp}, and recalling \eqref{eq:skor_exp}, we may compute $\skor(u)$ using~\eqref{eq:skor_exp}
\begin{align*}
\skor(u) = \Gamma\skor(v) - \int_0^TD_t\Gamma \cdot v_t dt
\end{align*}
where $\skor(v)$ is a standard It\'o integral. However, recall that we only have access to first $\Loss'(X_t)$ and second order $\Loss''(X_t)$ derivatives. We now introduce a simple trick. Naively computing $D_t\Gamma$ gives 
\begin{align*}
    D_t\Gamma = \Gamma\int_0^s\Loss'''(X_u)D_tX_udu
\end{align*}
which relies on $\Loss'''$. Instead, we show below that we can express $D_t\Gamma$ only in terms of $\Loss'(X_t)$, $L''(X_t)$.

\begin{lemma}
\label{lem:skor}
\begin{equation}
\label{eq:lem_skoru}
    \skor(u) = \Gamma\skor(v) - \int_0^TD_t\Gamma\cdot v_tdt 
\end{equation}
where 
\begin{align*}
\Gamma &= \exp\left(\int_0^s\Loss''(X_u)du\right)\\
D_t\Gamma &= \int_0^s \tilde{\Gamma}_u\left[d\Loss'(X_u)-\Loss''(X_u)dX_u \right]\\
v_t &= \frac{1}{\sqrt{2}T}\exp\left(-\int_0^t\Loss''(X_u)du\right)
\end{align*}

\end{lemma}

\begin{proof}
See Appendix~\ref{sec:pf_skor}.
\end{proof}

Thus, we may compute \eqref{eq:mall_reform} directly using \eqref{eq:lem_skoru} and \[D( \Loss'(X_s)),u\rangle = \int_0^s  \Loss''(X_s)D_tX_su_tdt\]
with $u_t$ in \eqref{eq:u_choice_2} and $D_tX_s$ in \eqref{eq:DtXs}. 


Next, we proceed from computation of gradient estimates at individual points $\alpha$, through \eqref{eq:mall_reform}, to a full passive learning procedure. This passive learning algorithm integrates these gradient estimates into the external Langevin chain \eqref{eq:inv_langevin} running on a slower time-scale, and converges to the Gibbs measure \eqref{eq:gibbs_measure} proportional to $\exp(-\beta\Loss(\cdot))$, thus enabling recovery of $\Loss$ through MCMC. 

\subsection{Passive Learning by Malliavin Gradient Estimation}

Now that we have constructed a principled way to evaluate the gradient $\nabla\Loss(\alpha)$ for any $\alpha$, from only mis-specified forward process data \eqref{eq:X_sps}, we integrate these estimates into a slow time-scale Langevin dynamics to recover $L(\cdot)$ asymptotically as the potential function of the stationary Gibbs measure. 

Algorithm~\ref{alg:mall_IRL} displays the full passive learning procedure, implemented in discrete-time with gradient estimates evaluated via Monte Carlo. The final loss estimate is recovered through the log-transform of empirical histogram $\hat{\pi}$ formed through MCMC samples. The following theorem demonstrates the asymptotic consistency of this reconstruction, i.e., that this discrete stochastic algorithm converges in an appropriate weak sense to the Langevin diffusion \eqref{eq:inv_langevin} which has \eqref{eq:gibbs_measure} as its stationary measure.

\begin{theorem}[Asymptotic consistency of Algorithm 1]
\label{thm:convg}
Assume that $L$ is $C^2$, and $L'$ is locally Lipschitz, and that
$L$ is confining in the sense that
\begin{equation}
\label{eq:confining}
L(\alpha)\to\infty \, \text{as } |\alpha|\to\infty, \,\,Z_\beta := \int_{\mathbb{R}} e^{-\beta L(a)}\,da < \infty.
\end{equation}
 Then, as $k\to\infty$, $\eta\to0$, and
$N\to\infty$, the empirical law of $\alpha_k$ converges to $\pi_\infty$. 
\end{theorem}
\begin{proof}
The outline follows the standard SGLD scheme. First, as
$N\to\infty$, the Malliavin Monte Carlo estimator $\widehat{\nabla L}(\alpha)$
converges to the true drift $\nabla L(\alpha)$. Hence the outer recursion in
Algorithm~1 is, for small $\eta$, an Euler--Maruyama discretization of the
Langevin diffusion
\[
d\alpha_t = -\nabla L(\alpha_t)\,dt + \sqrt{2\beta^{-1}}\,dW_t .
\]
Second, under the stated confining assumptions, this diffusion admits the Gibbs
measure $\pi_\infty(d\alpha)\propto e^{-\beta L(\alpha)}\,d\alpha$ as invariant
law, and converges to it as $t\to\infty$. Therefore, letting first
$N\to\infty$ and $\eta\to0$ and then $k\to\infty$, the empirical law of the
outer chain converges to $\pi_\infty$. See Appendix~\ref{sec:thmpf} for more details and comments on finite-sample analysis.
\end{proof}

Hence, Algorithm~\ref{alg:mall_IRL} asymptotically recovers $L(\alpha)$ through
\begin{equation}
\label{eq:Lprop}
L(\alpha) \propto -\beta^{-1}\log \pi_\infty(\alpha)
\end{equation}

In other words, the estimate $\hat{\Loss}(\cdot)$ recovered through Algorithm~\ref{alg:mall_IRL} is consistent, in that it converges to the true loss \eqref{eq:Lprop} as $k\to\infty, \eta\to 0, N\to\infty$.

\section{Numerical Example}
\label{sec:numerical_example}

We illustrate the passive inverse-learning procedure on the one-dimensional Langevin diffusion
\begin{equation}
\label{eq:lang_sim}
dX_t = -L'(X_t)\,dt + \sqrt{2}\,dW_t,
\end{equation}
with hidden potential
\begin{equation}
\label{eq:L}
L(x) = \frac{x^4}{4} + \frac{x^2}{2},
\quad
L'(x)=x^3+x,
\quad
L''(x)=3x^2+1.
\end{equation}
Specifically, we observe sample trajectories from \eqref{eq:lang_sim}, and recover the loss \eqref{eq:L} in its entirety, through our adaptive IRL approach.
We first validate the Malliavin conditional-gradient estimator \eqref{eq:mall_reform}
at a prescribed conditioning time \(s\), and then use these pointwise gradient estimates as inputs to the passive Langevin chain in Algorithm~1.

\subsection{Simulation of Malliavin Quantities.}
For each simulated path of \eqref{eq:lang_sim} we store only the observable quantities \eqref{eq:X_sps}. We fix a conditioning time $s=0.8$. For each simulated path, we evaluate the Malliavin derivative using the explicit formula \eqref{eq:DtXs}. We then construct the reciprocal weight \eqref{eq:u_choice_2}, which factorizes as \eqref{eq:u_decomp}, allowing us to compute the Skorohod weight via
\begin{equation*}
S(u)
=
\Gamma S(v)
-
\int_0^s D_t\Gamma\,v_t\,dt,
\end{equation*}
where \(S(v)=\int_0^s v_t\,dW_t\). The term \(D_t\Gamma\) is computed pathwise using the identity
\begin{equation}
L'''(X_t)\,dt = dL'(X_t)-L''(X_t)\,dX_t,
\end{equation}
so that no third derivative needs to be computed.

\subsection{Monte Carlo gradient estimator.}
For a fixed target location \(\alpha\), we estimate the numerator and denominator of the Malliavin ratio formula \eqref{eq:mall_reform} by Monte Carlo averaging over the simulated paths:
\begin{align*}
\widehat N(\alpha)
&=
\frac{1}{N}\sum_{i=1}^N
\mathbf 1_{\{X_s^{(i)}>\alpha\}}
\biggl(
L'(X_s^{(i)})S(u^{(i)})
\\&\qquad \qquad\qquad-
\left\langle D(L'(X_s^{(i)})),u^{(i)}\right\rangle
\biggr),\\
\widehat D(\alpha)
&=
\frac{1}{N}\sum_{i=1}^N
\mathbf 1_{\{X_s^{(i)}>\alpha\}}S(u^{(i)}),
\end{align*}
with $\left\langle D(L'(X_s)),u\right\rangle
=
\int_0^s L''(X_s)D_tX_su_t\,dt.$
The resulting gradient estimate is
\begin{equation}
\label{eq:grad_est}
\widehat{L'}(\alpha)=\frac{\widehat N(\alpha)}{\widehat D(\alpha)}.
\end{equation}

Now we validate the numerical evaluation of \eqref{eq:grad_est}. We first evaluate \(\widehat{L'}(\alpha)\) on a grid of target points \(\alpha\in[-1,1]\) and compare it to the true gradient $
L'(\alpha)=\alpha^3+\alpha$ \eqref{eq:L}. 
This provides a direct check of the accuracy of the Malliavin reformulation before embedding it in the slower passive Langevin chain. Figure~\ref{fig:grad_estimation_accuracy} compares this estimated gradient \(\widehat{L'}(\alpha)\) against the ground truth \(L'(\alpha)\) for $N=5000$ simulated Monte Carlo paths. We see that there is a slight bias across $\alpha$, due to the quotient form of the Malliavin gradient estimator \eqref{eq:mall_reform}, but that otherwise the estimate has small error and tracks the true gradient across the domain.

\subsection{MCMC Recovery of $L(\cdot)$}
\label{sec:outer_langevin_chain}

Having obtained pointwise estimates of the gradient \(L'(\alpha)\) from the Malliavin conditional expectation formula, we next embed these estimates into the slow-time-scale inverse-learning chain proposed in Algorithm~\ref{alg:mall_IRL}:
\begin{equation}
\alpha_{k+1}
=
\alpha_k - \eta\,\widehat{L'}(\alpha_k)
+
\sqrt{2\eta\beta^{-1}}\,\xi_k,
\quad
\xi_k \sim \mathcal N(0,1),
\label{eq:outer_langevin_discrete}
\end{equation}
where \(\eta>0\) is the outer step size and \(\widehat{L'}(\alpha_k)\) is the Monte Carlo Malliavin estimator \eqref{eq:grad_est} obtained from the forward trajectories at the current target location \(\alpha_k\). 

\begin{figure}[t]
    \centering
     \includegraphics[width=0.7\linewidth]{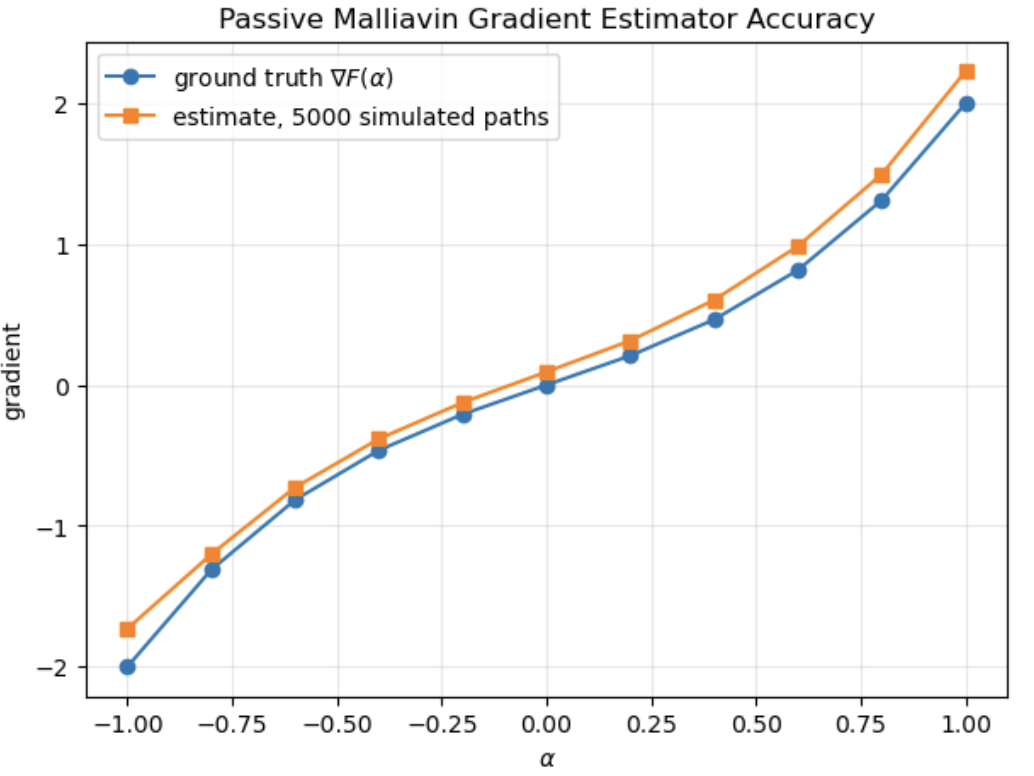}
    \caption{Estimated gradient \(\widehat{ L'}(\alpha)\) versus the true gradient \(L'(\alpha)\) over a grid of target locations \(\alpha\), using  Monte Carlo estimation with sample size \(N=5000\).}
    \label{fig:grad_estimation_accuracy}
\end{figure}

\begin{figure}
    \centering
     \includegraphics[width=1.0\linewidth]{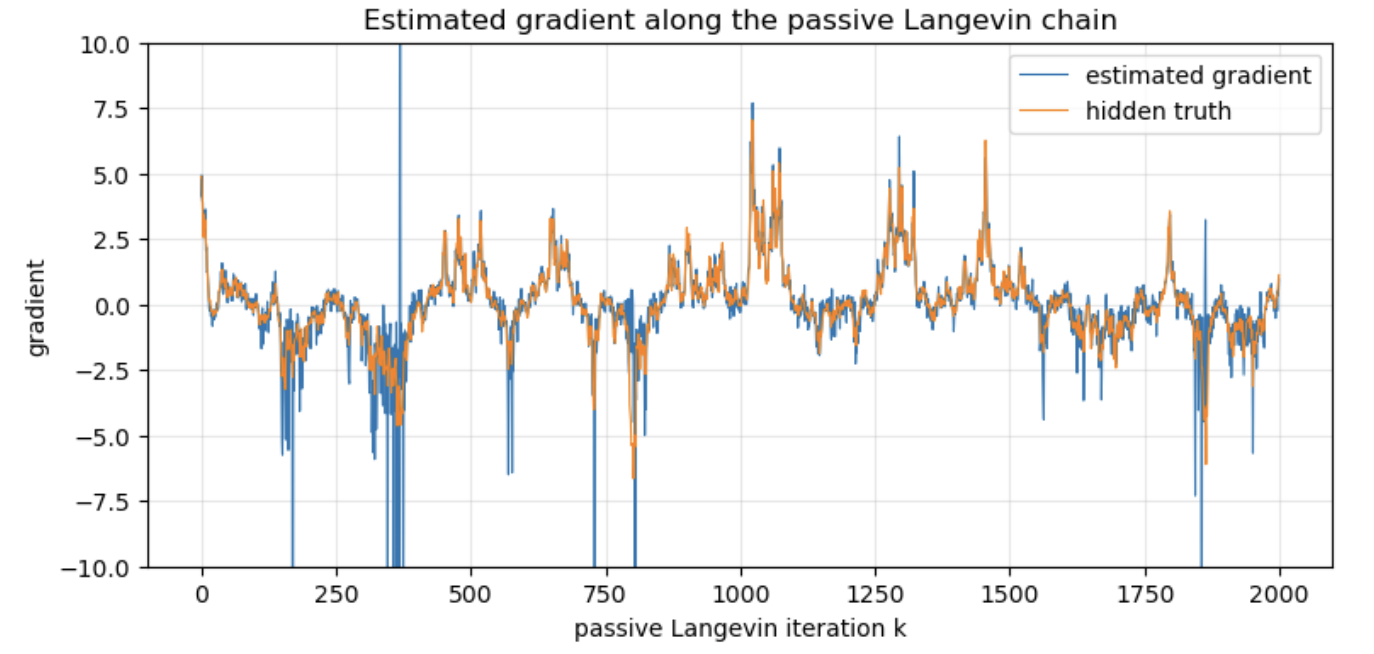}
    \caption{Sequential gradient estimates $\widehat{\Loss'}(\alpha)$ versus the true gradient $\Loss'(\alpha)$ across iterations of the outer passive Langevin chain \eqref{eq:outer_langevin_discrete}. Occasional outliers appear, yet the overall procedure recovers the target Gibbs distribution and hidden loss function accurately (see Figures~\ref{fig:outer_chain_gibbs} and \ref{fig:L_reconstruction})}
    \label{fig:outer_chain_grads}
\end{figure}

\begin{figure}
    \centering
     \includegraphics[width=0.75\linewidth]{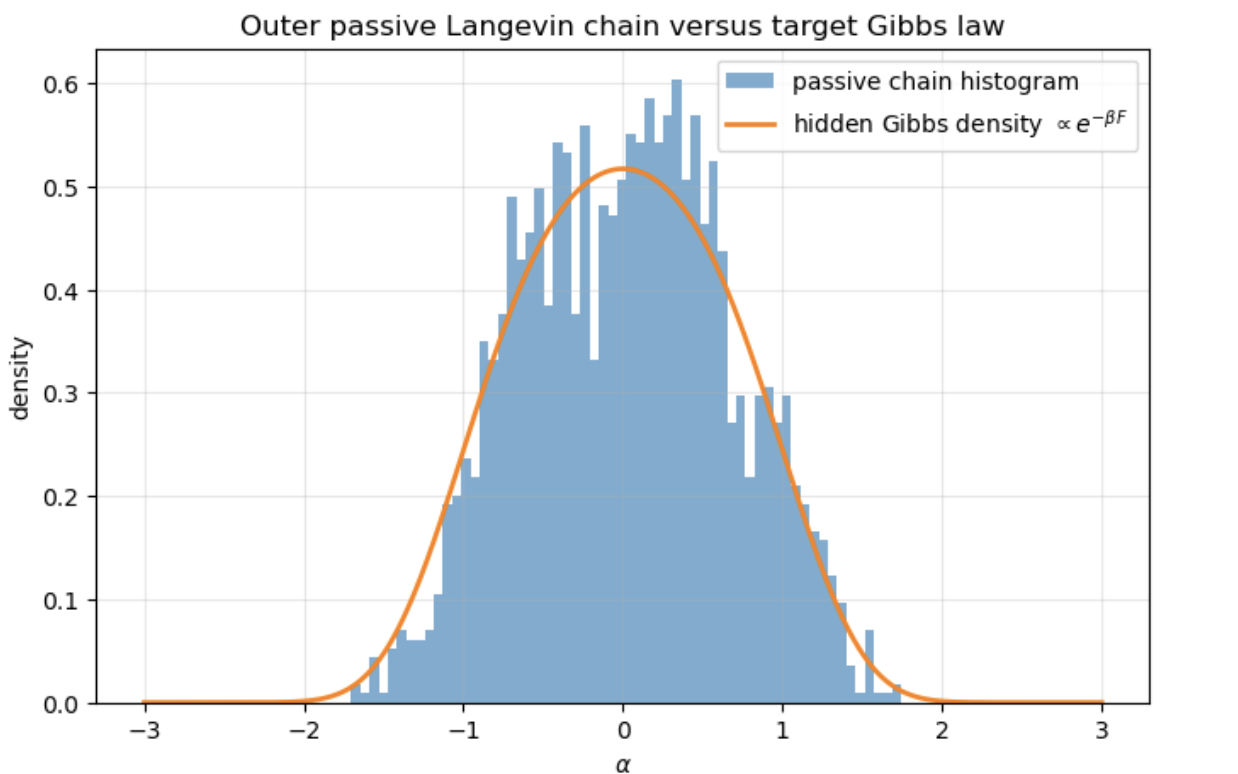}
    \caption{Empirical histogram of the stationary distribution of the Langevin chain \(\{\alpha_k\}\) ($2000$ MCMC steps after a $300$-step burn-in), compared with the target Gibbs density \(\pi_\beta(\alpha)\propto e^{-\beta L(\alpha)}\).}
    \label{fig:outer_chain_gibbs}
\end{figure}

\begin{figure}
    \centering
     \includegraphics[width=0.75\linewidth]{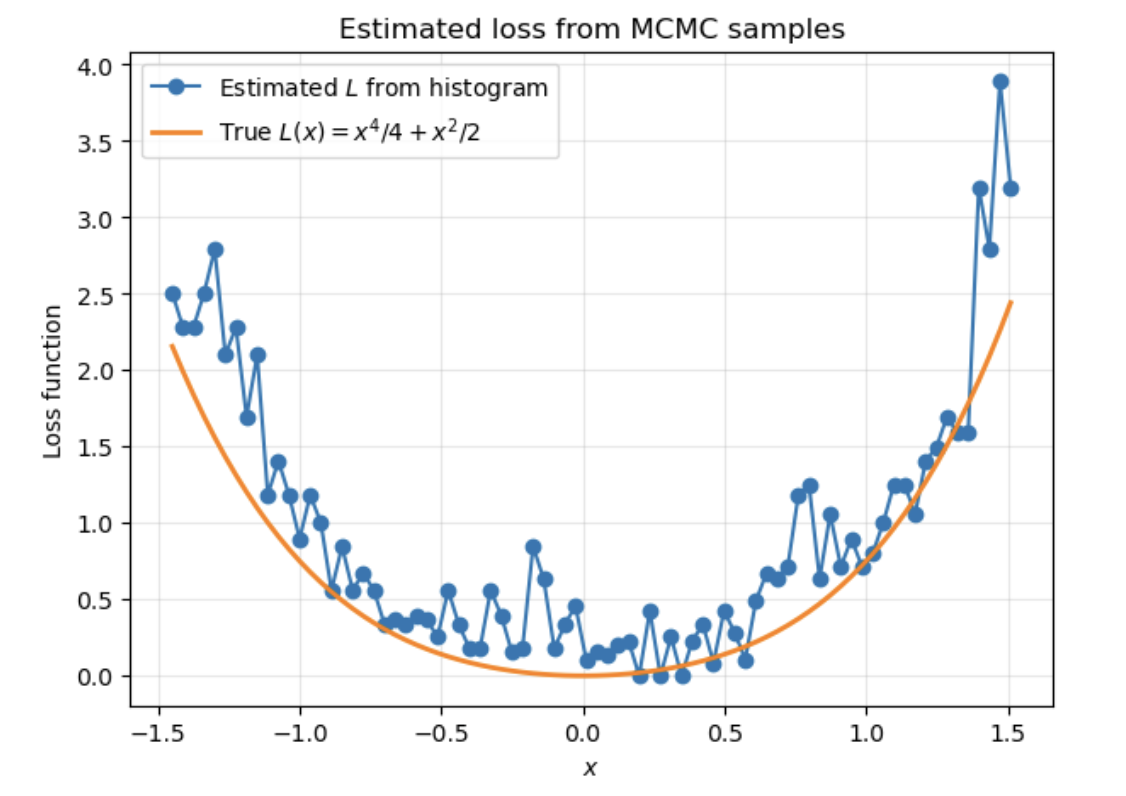}
    \caption{Adaptive IRL for reconstructing the loss function $\Loss(\cdot)$ from MCMC samples of $\exp(-\beta\Loss(\cdot)$. Specifically, $L(\cdot) = -\frac{1}{\beta}\log(\hat{\pi}(\cdot))$, where $\hat{\pi}(\cdot)$ is the empirical MCMC distribution in Fig.~\ref{fig:outer_chain_gibbs}.}
    \label{fig:L_reconstruction}
\end{figure}
We initialize the outer chain at $\alpha_0\in\mathbb R$, evolve
\eqref{eq:outer_langevin_discrete} for $2000$ iterations, discard the first
$300$ samples as burn-in, and retain the remaining iterates
$\{\alpha_k\}$ as approximate MCMC samples from the stationary law of the
inverse dynamics.

Figure~\ref{fig:outer_chain_grads} compares the estimated gradient
$\widehat{L'}(\alpha_k)$ with the true gradient $L'(\alpha_k)$ along the outer
Langevin trajectory. We observe that, for the majority of iterations, the
Malliavin estimator tracks the true derivative closely, with only occasional
outliers.

We next compare the empirical histogram of the retained samples with the target
Gibbs density
\begin{equation}
\label{eq:gibbs_density_explicit}
\pi_\beta(\alpha)\propto Z_\beta^{-1}\exp(-\beta L(\alpha)),
\end{equation}
where $L$ is given by \eqref{eq:L}. Figure~\ref{fig:outer_chain_gibbs} shows
that the empirical stationary distribution of the outer chain closely matches
the target Gibbs law. Finally, Figure~\ref{fig:L_reconstruction} displays the
loss reconstructed from the empirical histogram via
\[
\widehat L(\alpha) = -\beta^{-1}\log \widehat \pi(\alpha),
\]
up to an additive constant. The reconstructed potential agrees closely with the
true loss, confirming that the proposed passive learning procedure successfully
recovers the underlying objective from observed stochastic gradient dynamics.

\section{Conclusion}
We developed a Malliavin calculus-based framework for adaptive inverse reinforcement learning. The key tool is a reformulation of the counterfactual gradient required for passive IRL as a ratio of unconditional expectations. This avoids the measure-zero conditioning problem that plagues direct Monte Carlo and kernel-based methods. For Langevin forward dynamics, the requisite Malliavin derivative and Skorohod integral terms admit explicit expressions in terms of observed trajectories, producing a practical estimator that depends only on first- and second-order pathwise derivatives. Embedding this estimator within an outer Langevin chain yields a fully passive learning algorithm whose stationary distribution recovers the unknown loss. Numerical experiments confirm accurate gradient estimation, faithful recovery of the target Gibbs distribution, and effective reconstruction of the hidden loss function. These results establish Malliavin-based counterfactual estimation as an efficient and scalable alternative to kernel-based approaches for passive IRL.

{\bf Note}. The numerical results are fully reproducible with open-source Python code  available at \texttt{https://github.com/LukeSnow0/}.

\bibliographystyle{IEEEtran}
\bibliography{vkm.bib}

\section{Appendix}

\subsection{Supporting Lemmas}

\begin{lemma}
\label{lem:F3}
$\Loss'''(X_t)dt = d\Loss'(X_t) - \Loss''(X_t)dX_t$
\end{lemma}

\subsection{Proofs}
\subsubsection{Proof of Lemma~\ref{lem:DXs}}
\label{sec:pf_DXs}
Let $Y_s:=\nabla_x X_s$. Differentiating \eqref{eq:langevin_recall} with respect
to the initial condition gives
\[
Y_s = 1-\int_0^s \nabla^2\Loss(X_u)Y_u\,du,
\]
or equivalently
\[
\frac{d}{ds}Y_s=-\nabla^2\Loss(X_s)Y_s,
\qquad Y_0=1.
\]
Hence
\[
Y_s=\exp\!\left(-\int_0^s \nabla^2\Loss(X_u)\,du\right).
\]

\subsubsection{Proof of Lemma~\ref{lem:F3}}

It\^o's Lemma tells us that 
\[d\Loss'(X_t) = (\Loss'''(X_t) - \Loss'(X_t)\Loss''(X_t))dt + \sqrt{2}\Loss''(X_t)dW_t\]
and furthermore we may write $dW_t$ as 
\[dW_t = (dX_t + \Loss'(X_t)dt)/\sqrt{2}\] by definition of the forward Langevin dynamics. Thus, substituting and rearranging terms gives us
\[\Loss'''(X_t)dt = d\Loss'(X_t) - \Loss''(X_t)dX_t\]

\subsubsection{Proof of Lemma~\ref{lem:skor}}
\label{sec:pf_skor}

By Lemma~\ref{lem:F3}, we may write out $D_t\Gamma$ as 
\begin{align*}
D_t\Gamma &= \Gamma\int_0^sD_tX_u[d\Loss'(X_u) - \Loss''(X_u)dX_u] \\
& = \sqrt{2}\Gamma\int_0^s\exp\left(-\int_t^u\Loss''(X_\gamma)d\gamma\right)[d\Loss'(X_u) \\
&\qquad \qquad - \Loss''(X_u)dX_u]
\end{align*}
Thus, 
\begin{align}
\begin{split}
\label{eq:skor_expression}
&\skor(u) = \Gamma\skor(\tilde{u}) - \int_0^TD_t\Gamma\cdot\tilde{u}_tdt \\
&= \exp\left(\int_0^s\Loss''(X_u)du\right)\cdot\\&\qquad \int_0^T\frac{1}{\sqrt{2}T}\exp\left(-\int_0^t\Loss''(X_u)du\right)dW_t \\
& \quad  - \int_0^T\Gamma\int_0^s\exp\left(-\int_t^u\Loss''(X_\gamma)d\gamma\right)[d\Loss'(X_u) \\&\qquad - \Loss''(X_u)dX_u]
 \cdot\frac{1}{T}\exp\left(-\int_0^t\Loss''(X_u)du\right)dt
\end{split}
\end{align}

\subsubsection{Proof of Theorem~\ref{thm:convg}}
\label{sec:thmpf}
The main idea is as follows. By the law of large numbers, the empirical
numerator and denominator in Algorithm~\ref{alg:mall_IRL} converge almost surely
to their population counterparts, so the Malliavin estimator is consistent.
Substituting this estimator into the outer update yields an Euler--Maruyama
scheme,
\[
\alpha_{k+1}
=
\alpha_k-\eta\,\widehat{\nabla L}_N(\alpha_k)
+\sqrt{2\beta^{-1}\eta}\,\xi_{k+1},
\,\,\, \xi_{k+1}\sim N(0,1),
\]
whose continuous-time interpolation converges weakly, as $N\to\infty$ and
$\eta\to0$, to \eqref{eq:inv_langevin} by standard stochastic approximation
results \cite{KY03}. The stationary Fokker--Planck equation for
\eqref{eq:inv_langevin} is
\[
0=\frac{d}{d\alpha}\!\bigl(L'(\alpha)\rho(\alpha)\bigr)
+\beta^{-1}\frac{d^2}{d\alpha^2}\rho(\alpha),
\]
and direct substitution shows that
\[
\rho(\alpha)=Z_\beta^{-1}e^{-\beta L(\alpha)}
\]
is a solution. The confining assumption \eqref{eq:confining} ensures
normalizability, so this defines the invariant measure $\pi_\infty$. Hence the
limiting diffusion converges to $\pi_\infty$, and therefore the empirical law
of the outer chain converges to $\pi_\infty$ as $k\to\infty$.

\medskip
\noindent\textbf{Finite-sample perspective.}
A corresponding finite-sample analysis can follow the techniques in
\cite{SK25,RRT17a}. Writing $\mu_k:=\CL(\alpha_k)$ for the discrete outer
chain law, $\nu_{k\eta}:=\CL(\alpha(k\eta))$ for the limiting diffusion law, and
$\pi_\infty$ for its Gibbs stationary law, one decomposes the 2-Wasserstein distance
\[
W_2(\mu_k,\pi_\infty)
\le
W_2(\mu_k,\nu_{k\eta})+W_2(\nu_{k\eta},\pi_\infty).
\]
The first term is the diffusion-approximation error, controlled by comparing the
Euler scheme to the limiting diffusion; the second is the diffusion-convergence
error, controlled by ergodicity of the limiting diffusion, e.g.\ via
log-Sobolev and entropy-decay arguments. In the asymptotic regime these two
terms vanish, yielding $\mu_k\Rightarrow\pi_\infty$ and hence
\[
L(\alpha)\propto -\beta^{-1}\log \pi_\infty(\alpha).
\]

\end{document}